\documentclass[9.5bp,journal,compsoc]{IEEEtran}
\usepackage{epsfig}
\usepackage{graphicx}
\usepackage[pagebackref=false,breaklinks=true,colorlinks,linkcolor={black},citecolor={black},urlcolor={blue},bookmarks=false]{hyperref}
\usepackage{cite}
\usepackage{microtype}
\usepackage{subfigure}
\usepackage{url}

\usepackage{array}
\usepackage[normalem]{ulem}
\usepackage{ulem}
\usepackage{amsmath}
\usepackage{graphicx,subfigure}
\usepackage{longtable}
\usepackage{rotating}
\usepackage{multirow}
\usepackage{makecell}
\usepackage{bm}
\usepackage{float}
\usepackage{flushend}
\usepackage{xcolor}
\usepackage{soul}
\usepackage{booktabs}
\usepackage{threeparttable}
\usepackage{color}
\usepackage{algorithm}
\usepackage{algorithmic}
\usepackage{tabularx}
\usepackage{amsfonts}
\usepackage{epstopdf}
\usepackage{latexsym}
\usepackage{amssymb}
\usepackage{footnote}

\usepackage{pifont}

\newtheorem{theorem}{Theorem}
\newtheorem{lemma}{Lemma}
\newtheorem{proof}{Proof}
\newtheorem{definition}{Definition}
\newtheorem{remark}{Remark}
\renewcommand{\IEEEQED}{\IEEEQEDopen}

\usepackage{threeparttable}

\def\Sp{{\scriptsize{\textcircled{{\emph{\tiny{\textbf{Sp}}}}}}}}

\begin{document}

\title{Interpretable Multi-View Clustering Based on \\
            Anchor Graph Tensor Factorization}

\author{Rui Wang,
        Jing Li,
        Quanxue~Gao,
        Cheng~Deng
\IEEEcompsocitemizethanks{
\IEEEcompsocthanksitem This work was supported in part by xxxxx. (Corresponding author: Q. Gao.)\protect
\IEEEcompsocthanksitem R. Wang, J. Li, Q. Gao are with the State Key laboratory of Integrated Services Networks, Xidian University, Xi'an 710071, China (e-mail: qxgao@xidian.edu.cn).\protect

\protect}%
\thanks{Manuscript received XXXX; revised XXXX; accepted XXXX.}}

\markboth{IEEE TRANSACTIONS}%
{Shell \MakeLowercase{\textit{Xia et al.}}: Multi-View Clustering via Semi-non-negative Tensor Factorization}

\IEEEtitleabstractindextext{%
\begin{abstract}
The clustering method based on the anchor graph has gained significant attention due to its exceptional clustering performance and ability to process large-scale data. One common approach is to learn bipartite graphs with K-connected components, helping avoid the need for post-processing. However, this method has strict parameter requirements and may not always get K-connected components. To address this issue, an alternative approach is to directly obtain the cluster label matrix by performing non-negative matrix factorization (NMF) on the anchor graph. Nevertheless, existing multi-view clustering methods based on anchor graph factorization lack adequate cluster interpretability for the decomposed matrix and often overlook the inter-view information. We address this limitation by using non-negative tensor factorization to decompose an anchor graph tensor that combines anchor graphs from multiple views. This approach allows us to consider inter-view information comprehensively. The decomposed tensors, namely the sample indicator tensor and the anchor indicator tensor, enhance the interpretability of the factorization. Extensive experiments validate the effectiveness of this method.
\end{abstract}

}

\maketitle
\IEEEdisplaynontitleabstractindextext
\IEEEpeerreviewmaketitle

\IEEEraisesectionheading{\section{Introduction}\label{sec:introduction}}
\IEEEPARstart{C}{lustering}, a fundamental technique in unsupervised learning, is frequently applied to categorize samples within datasets into distinct groups or clusters. In the burgeoning field of big data, the acquisition of vast volumes of unlabeled data has become increasingly straightforward. Clustering plays a pivotal role in uncovering the intrinsic structures and patterns in such data, thereby facilitating the assignment of labels to unlabeled data, which is crucial for enhancing the utility of these data in downstream tasks \cite{cai2010graph,huang2020regularized,mei2023joint,zhao2023contrastive}. Consequently, clustering has emerged as a prominent area of research in machine learning.

In an era marked by growing data diversity and complexity, individual objects are often characterized by multiple attribute sets, termed as multi-view data. The prevalence of multi-view data in the big data landscape is unmistakable. For example, multimedia can be holistically represented using textual, image, audio, and video features. Similarly, a single image can be analyzed using various descriptors such as scale-invariant feature transformations (SIFT), histograms of oriented gradients (HOG), and local binary patterns (LBP). In automated driving, the integration of heterogeneous sensor data, including RGB camera, infrared camera, and LIDAR, is imperative \cite{yang2018multi,chao2021survey}.
To enhance clustering in such contexts, multi-view clustering (MVC) methods have gained significant traction. These methods seek to amalgamate different data views, providing a more holistic representation than traditional single-view approaches, often resulting in superior performance.

Among various MVC technologies, multi-view clustering based on graph stands out as a benchmark approach. However, the graph-based methods involve $n \times n$ graph construction and eigen-decomposition of Laplacian matrix whose computational complexity are $\mathcal{O}(Vn^2)$ and $\mathcal{O}(n^3)$, respectively. It limits the application on large-scale multi-view data.
To end this, a myriad of anchor graph-based methods have been proposed \cite{li2020multiview,yang2022multi,xia2023TBGL}.
Anchor graph can well represent multi-view data by modeling the relationship between $n$ samples and $m$ anchors. It can reduce the computational complexity and storage complexity due to $m \ll n$.

Generally speaking, to avoid post-processing, the goal of some multi-view clustering methods based on anchor graphs is to learn a bipartite graph with K-connected components and obtain the clustering labels from the connected components. However, this method has high parameter requirements and may not be able to find K-connected components in some cases.

To address this, another idea is to perform non-negative matrix factorization (NMF) on anchor graph to obtain the cluster indicator matrix directly. Nonetheless, current multi-view NMF-based clustering methods encounter several challenges:
\begin{enumerate}
  \item Independently performing NMF on each view may fail to maintain the spatial structures inherent across views in the original data.
  \item Some methods use non-negative constraints on the clustering indicator matrix obtained from decomposition only, leading to a lack of interpretability.
\end{enumerate}

To overcome these limitations, we propose an interpretable multi-view clustering model based on Anchor Graph Tensor Factorization (AGTF). Our approach extends NMF to operate directly on third-order tensors, overcoming the traditional limitation of NMF to second-order matrices. This extension allows for direct processing of third-order tensors, preserving more intrinsic spatial structure information across different views. Additionally, we provide meaningful interpretations for the two indicator tensors derived from the tensor factorization, enhancing the process's interpretability. To derive a common consensus indicator matrix from different views, we utilize the tensor Schatten p-norm to impose a low-rank constraint on the indicator tensor, effectively capturing the complementary information between views.

Our extensive experiments validate the superior clustering performance of our model.

The primary contributions of this paper are summarized as follows:
\begin{itemize}
  \item We extend NMF to Nonnegative Tensor Factorization (NTF), enabling direct decomposition of third-order tensors, and interpret the tensors resulting from this decomposition.
  \item We introduce the tensor Schatten p-norm to exploit complementary information across different views, facilitating the derivation of a common consensus indicator matrix.
  \item We present an optimization algorithm for NTF, demonstrating its convergence to the KKT stationary point both mathematically and experimentally.
\end{itemize}

\section{Related Works}\label{Related work}
\subsection{MVC Based on Anchor Graph}
Multi-view clustering (MVC) method based on anchor graph is proposed to approach the problem that graph based methods can not handle large-scale multi-view data. Compared with the $n \times n$ similarity graph, the $n \times m (m \ll n)$ anchor graph can obviously reduce the computational complexity and storage complexity.

A partially anchor graph-based approach is to perform some clustering strategies, such as spectral clustering and subspace clustering, on the constructed anchor graph to obtain clustering results. But this kind of post-processing is time-consuming, many multi-view clustering methods based on anchor graph without post-processing have been proposed gradually.

In order to avoid post-processing, it is more common to construct bipartite graphs with anchor graphs, then learn bipartite graphs with K-connected components, and finally obtain clustering labels through connected components.
MVSC \cite{li2015large} constructs sub-bipartite-graphs in each view using raw data points and anchors, and uses local manifold fusion to fuse the information from the individual sub-graphs. Clustering labels of anchors are used to handle the out-of-sample problems.
MVGL \cite{zhan2018graph} learns a global graph from different single view graphs. It avoids post-processing by using connected components to obtain clustering metrics, but it may not be able to handle large-scale data.
SFMC \cite{li2020multiview} proposes a parametr-free multi-view cluster graph fusion framework for obtaining compatible union graphs by self-supervised weighting, and it uses connected components to represent clusters.
MSC-BG \cite{yang2022multi} uses the Schatten $p$-norm to explore complementary information between different views and obtains clusters by K-connecting components.
TBGL \cite{xia2023TBGL} uses the Schatten $p$-norm to explore similarities of inter-view and combines the $\ell_{\textrm{1,2}}$-norm minimization regularization and connectivity constraints to explore the similarity of intra-view.

Although these methods can avoid post-processing with connected components, they are very demanding on parameters and may not find K-connected components in some cases.

\subsection{MVC Based on NMF}
Non-negative matrix factorization (NMF) is a common data analysis method.
To transition from a single-view NMF to a multi-view approach, a straightforward strategy is to align the indicator matrices from different views, ensuring similarity or identity. This concept forms the basis of many existing NMF-based Multi-View Clustering (MVC) methods.

In their effort to amalgamate multi-view data within the NMF paradigm, Akata et al. \cite{akata2011non} enforced identical clustering index matrices across different views via NMF decomposition. However, recognizing the often incomparable scales among views, Liu et al. \cite{liu2013multi} proposed performing NMF independently on each view, followed by aligning the resultant clustering index matrices to establish cross-view comparability.

Building upon this NMF-based MVC framework, several enhancements have been introduced. Cai et al. \cite{cai2010graph}, acknowledging NMF's limitation to only achieve consistent data representation without preserving local geometric structures, incorporated a graph regularization term in the single-view context. This concept was further extended to multi-view scenarios by Wang et al. \cite{wang2018multiview} and Gao et al. \cite{gao2019multi}. Additionally, Deng et al. \cite{deng2023multi} introduced Tikhonov regularization to this framework, enhancing its reliability.

Owing to their ability to bypass the construction of similarity graphs, NMF-based MVC methods are particularly adept at managing large-scale multi-view data, unlike their graph-based counterparts. However, the efficiency of factorization is compromised in high-dimensional data, as these methods conventionally apply NMF directly to the original data. Addressing this, Yang et al. \cite{yang2021fast} integrated the concept of anchor graphs \cite{liu2010large,li2015large,kang2020large,li2020multiview} within the NMF framework. They proposed using fused-view anchor graphs for NMF, rather than the raw data, and further refined this approach in subsequent studies \cite{yang2022EMKCMAG}, focusing on consensus representation learning from multiple anchor graphs. Yang et al.'s later works \cite{yanga2022ERMCAGR,yangb2022ECMCAGE} broadly applied the anchor graph concept to NMF, significantly enhancing algorithmic efficiency.
Moreover, as anchor graphs in each view are structured as $\mathbb{R}^{n \times m}$ (where $n$ and $m$ represent the number of samples and anchors, respectively), the importance of performing NMF on these graphs is underscored. Under suitable constraints, the resultant matrices from NMF can be interpreted as sample and anchor indicator matrices \cite{ding2006orthogonal}.

However, the aforementioned methods somewhat overlook the inter-view spatial structure present in multi-view data, which can be advantageous for clustering. Therefore, refining MVC approaches based on NMF to incorporate this aspect remains a promising research direction.

\section{Notations}\label{Notations}
In this section, we present the notations adopted throughout this paper, along with detailed explanations of the tensor operations employed. A comprehensive list of these notations and their corresponding meanings is provided in Table \ref{notations}.

\begin{table}[h]
\caption{Notations and Descriptions.}
\label{notations}
\centering
\begin{tabular}{cc}
\toprule[2pt]
Notation                                & Descriptions \\
\midrule[0.5pt]
$\bm{\mathcal{H}}$                      & 3rd-order tensors \\
${\mathbf{H}}$                          & matrices \\
${\bf{h}}$                              & vectors \\
${h_{ijk}}$                             & the entries of $\bm{\mathcal{H}}$ \\
${\mathbf {H}}^{(i)}$                   & the $i$-th frontal slice of ${\bm{\mathcal {H}}}$ \\
$\overline {{\bm{\mathcal {H}}}}$       & \makecell[c]{the discrete Fourier transform (DFT) \\
                                                    of ${\bm{\mathcal {H}}}$ along the third dimension} \\
$\mathrm{tr}(\mathbf{H})$               & the trace and transpose of matrix $\mathbf{H}$ \\
${\left\| {\bm{\mathcal H}}\right\|_F}$ & the F-norm of ${\bm{\mathcal H}}$ \\
\midrule[2pt]
\end{tabular}
\end{table}

We next introduce the tensor t-product and tensor Schatten $p$-norm as detailed in Definitions \ref{def:t-prod} and \ref{tensorSpNorm}, respectively:

\begin{definition}[t-product \cite{kilmer2011factorization}]
\label{def:t-prod}
Let ${\bm{\mathcal{A}}}\in\mathbb{R}^{n_1\times m\times n_3}$ and ${\bm{\mathcal{B}}}\in \mathbb{R}^{m\times n_2\times n_3}$. The t-product ${\bm{\mathcal{A}}}*{\bm{\mathcal{B}}}\in\mathbb{R}^{n_1\times n_2\times n_3}$ is defined as
\begin{align*}
    {\bm{\mathcal{A}}}*{\bm{\mathcal{B}}} = \mathrm{ifft}(\mathrm{bdiag}(\overline{\mathbf A}\overline{\mathbf B}),[\ ],3),
\end{align*}
where $\overline{\mathbf A}=\mathrm{bdiag}(\bm{\overline{\mathcal{A}}})$ represents the block diagonal matrix, with its blocks being the frontal slices of $\bm{\overline{\mathcal{A}}}$.
\end{definition}

\begin{definition}\label{tensorSpNorm}~\cite{gao2020enhanced}
For a given tensor ${\bm{\mathcal H}}\in{\mathbb{R}}^{n_1 \times n_2 \times n_3}$ with $h=\min(n_1,n_3)$, the tensor Schatten $p$-norm of  ${\bm{\mathcal H}}$ is defined as:
\begin{equation}
\begin{array}{c}
{\left\| {\bm{\mathcal H}} \right\|_{{\Sp}}} = {\left( {\sum\limits_{i = 1}^{{n_2}} {\left\| {{{\bm{\overline {\cal H} }^{(i)}}}} \right\|_{{\Sp}}^p} } \right)^{\frac{1}{p}}} = {\left( {\sum\limits_{i = 1}^{{n_2}} {\sum\limits_{j = 1}^h {{\sigma_j}{{\left( {{\bm{\overline {\cal H} }^{(i)}}} \right)}^p}} } } \right)^{\frac{1}{p}}},
\end{array}
\label{4}
\end{equation}
where $0 < p \leqslant 1$, and ${\sigma_j}(\overline{\bm{\mathcal H}}^{(i)})$ denotes the j-th singular value of $\overline{\bm{\mathcal H}}^{(i)}$.
\end{definition}

It is important to note that for $0 < p \leqslant 1$, when $p$ is judiciously chosen, the Schatten $p$-norm can significantly enhance the approximation accuracy of the rank function, as discussed in~\cite{zha2020benchmark,xie2016weighted}.

\begin{figure*}[htbp]
	\centering
    \setlength{\abovecaptionskip}{-5mm}
    \begin{minipage}[c]{0.72\linewidth}
       \centering
	   \includegraphics[width=1.0\linewidth]{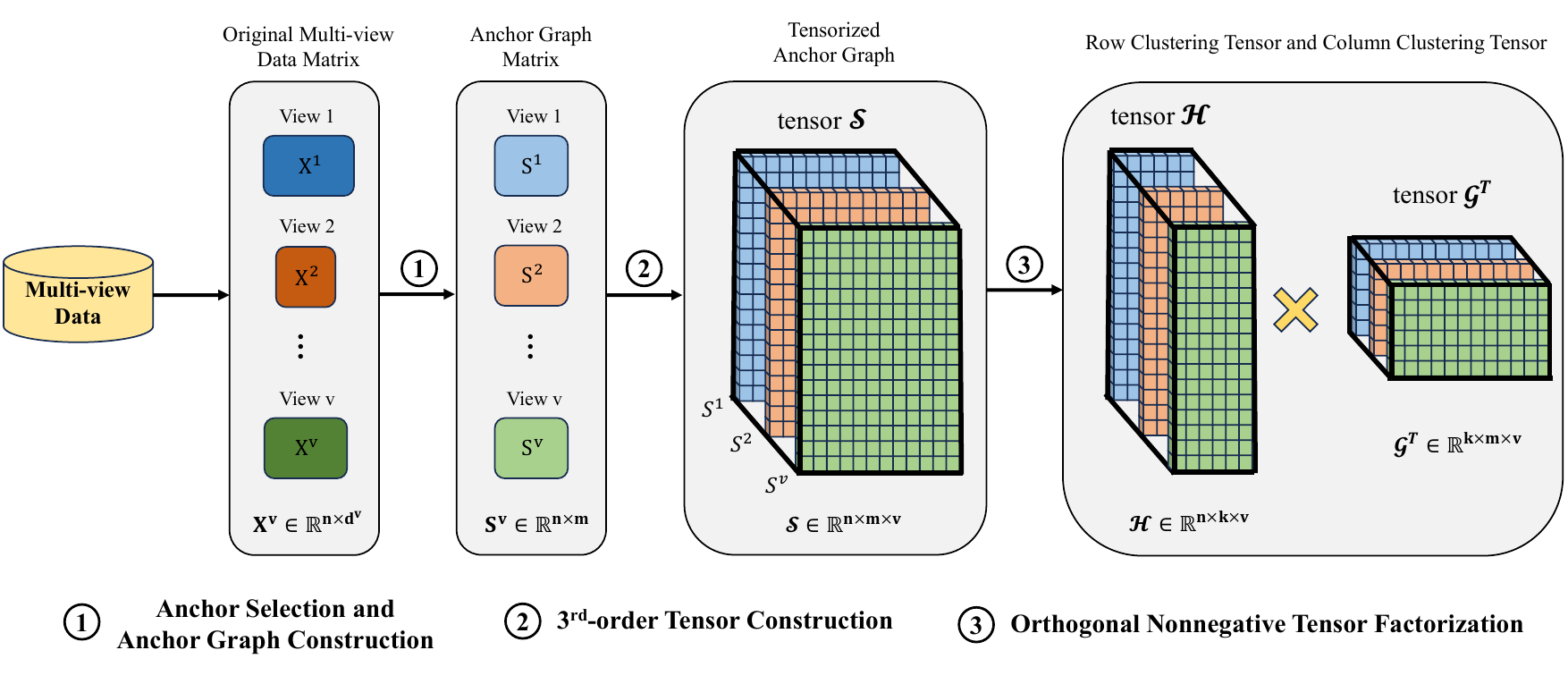}
	   \caption{Process flow of our model}
	   \label{process}
    \end{minipage}
    \hspace{10mm}
	\begin{minipage}[c]{0.20\linewidth}
	   \centering
	   \includegraphics[width=1.0\linewidth]{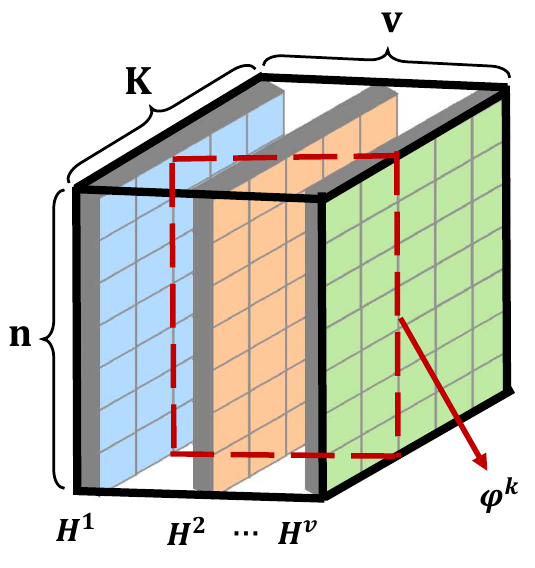}
	   \caption{Interpretation of the tensor Schatten $p$-norm}
	   \label{tensor}
	\end{minipage}
    \vspace{-5mm}
\end{figure*}

\section{Methodology}\label{Methodology}

\subsection{Motivation and Objective}
As described in related work, learning bipartite graphs with K-connected components for clustering has some limitations.
Inspired by Yang et al. \cite{yang2022EMKCMAG}, we integrate the concept of an anchor graph into the Non-negative Matrix Factorization (NMF) framework to obtain the cluster indicator matrix directly. This leads us to the following formulation:
\begin{equation}\label{SemiNmfAnchor}
    \min_{\mathbf{H} \geqslant 0}{\left\| {\mathbf S} - {\mathbf{H}} {\mathbf{G}}^{\mathrm{T}} \right \|}_F^2, \quad\textrm{s.t.}  \quad \mathbf{H}^{\mathrm{T}} \mathbf{H}=\mathbf{I},
\end{equation}
where $\mathbf S \in \mathbb{R}^{n \times m}$ denotes the anchor graph, encapsulating the relationships between $n$ samples and $m$ anchors. The methodology for selecting anchors and constructing the anchor graph can be found in \cite{xia2023TBGL}.

$\mathbf{H}$ in (\ref{SemiNmfAnchor}) can be considered as a row clustering indicator matrix according to \cite{ding2006orthogonal,ding2008convex}. Due to its non-negative orthogonal constraint, we can easily interpret it, i.e., each row of $\mathbf{H}$ has only one non-zero value, and the position of that value can indicate the cluster to which the sample shown in this row belongs. The $\mathbf{G}$ obtained by decomposing the anchor graph $\mathbf{S}$ can be regarded as the cluster indicator matrix of columns, i.e., the cluster indicator matrix of anchors. We consider constraining it to make it more interpretable, as follows:
\begin{equation}
    \begin{aligned}
        &\min {\left\| {\mathbf S} - {\mathbf{H}} {\mathbf{G}}^{\mathrm{T}} \right \|}_F^2, \\
        \quad\textrm{s.t.}  \quad \mathbf{H} \geqslant 0, &\mathbf{H}^{\mathrm{T}} \mathbf{H}=\mathbf{I}, \mathbf{G} \geqslant 0, \mathbf{G} \cdot \textbf{1} = \textbf{1}
    \end{aligned}
\end{equation}
Under this constraint, each row of $\mathbf{G}$ represents an anchor data, each element in the row can be considered as the probability that the anchor belongs to the corresponding cluster, and we believe that the position of the largest value can indicate the cluster to which the anchor belongs.

In Multi-view Clustering (MVC), the conventional NMF-based MVC framework applies NMF separately to each view, subsequently aligning the resulting indicator matrices from different views for consistency, as expressed in the equation:
\begin{equation}
    \min \sum_{v=1}^{V} {\left\| {\mathbf{S}^{(v)}} - {\mathbf{H}^{(v)}} {\mathbf{G}^{(v)}}^{\mathrm{T}} \right \|}_F^2 + \sum_{v=1}^{V} \lambda_v {\left\| {\mathbf{H}^{(v)}} - {\mathbf{H}^{*}} \right \|}_F^2,
\end{equation}

However, applying Non-negative Matrix Factorization (NMF) independently to different views results in the loss of crucial spatial structure information among multi-view data. This information is pivotal for improving clustering performance. Therefore, our aim is to effectively harness this information.

Consequently, we propose using a third-order tensor for representing multi-view data and opt for Non-negative Tensor Factorization (NTF) over NMF. This approach is intended to maximally preserve the spatial structure information across different views.

As previously mentioned, the tensors $\bm{\mathcal H}$ and $\bm{\mathcal G}$, derived from NTF, serve as the sample and anchor indicator tensors, respectively. To amalgamate information from multiple views, aligning the indicator matrices for various views becomes our subsequent step. Drawing inspiration from the remarkable efficacy of the tensor Schatten $p$-norm \cite{gao2020enhanced}, we impose a tensor low-rank constraint on $\bm{\mathcal H}$ and $\bm{\mathcal G}$ using this norm. Details about the tensor Schatten $p$-norm are provided in Remark \ref{rTensorSp}.

Our formulated objective function is:
\begin{equation}\label{of}
    \begin{aligned}
        &\min {\left\| \bm{\mathcal S} - \bm{\mathcal H} * \bm{\mathcal G}^{\mathrm{T}} \right \|}_F^2  + \lambda_1 {\left \|{\bm{\mathcal H}}\right \|}_\Sp^p + \lambda_2 {\left \|{\bm{\mathcal G}}\right \|}_\Sp^p \\
        & \emph{\textrm{s.t.}} \quad  \bm{\mathcal H}\geqslant 0, \bm{\mathcal H}^{\mathrm{T}} * \bm{\mathcal H} = \bm{\mathcal I}, \mathbf{G}^{(v)} \geqslant 0, \mathbf{G}^{(v)} \cdot \textbf{1} = \textbf{1}
    \end{aligned}
\end{equation}
where $0<p \leqslant 1$, and $\lambda_1$, $\lambda_2$ are the hyperparameters for the Schatten $p$-norm term. The matrix $\mathbf{G}^{(v)}$ represents the $v$-th frontal slice of the tensor $\bm{\mathcal{G}}$ as shown in Fig \ref{process}.

Figure \ref{process} illustrates the workflow of our proposed model.

\begin{remark}[Explanation of the tensor Schatten $p$-norm]\label{rTensorSp}
    Consider a matrix $\mathbf{A}$ with its singular values $\sigma_1, \cdots, \sigma_h$ arranged in descending order. The Schatten $p$-norm of $\mathbf{A}$ is defined as $\| \mathbf{A} \|_\Sp^p = \sigma_1^p + \cdots + \sigma_h^p$. As $p$ approaches 0, $\lim_{p \rightarrow 0} \| \mathbf{A} \|_\Sp^p$ equals the rank of $\mathbf{A}$, which is the count of non-zero singular values. Thus, the Schatten $p$-norm minimization, especially for $p<1$, potentially achieves a rank closer to the desired target than the conventional nuclear norm ($p=1$).

    For the tensor $\bm{\mathcal{H}}$, as depicted in Fig \ref{tensor}, its $k$-th lateral slice $\Phi^k$ represents the relationship of $n$ samples with the $k$-th cluster across different $v$ views. The goal in multi-view clustering is to harmonize the sample-cluster relationships in different views, making $\mathbf{H}_{:,k}^1, \cdots, \mathbf{H}_{:,k}^v$ as congruent as possible. However, the clustering structures often vary significantly across views. Applying the tensor Schatten $p$-norm to $\bm{\mathcal{H}}$ ensures that $\Phi^k$ maintains a spatially low-rank structure, leveraging complementary information across views and fostering consistency in the clustering indicators.
\end{remark}


\subsection{Optimization}
Utilizing the Augmented Lagrange Multiplier (ALM) method, we introduce three auxiliary variables $\bm{\mathcal Q}$, $\bm{\mathcal J}$, and $\bm{\mathcal F}$, setting $\bm{\mathcal H} = \bm{\mathcal Q}$, $\bm{\mathcal H} = \bm{\mathcal J}$, and $\bm{\mathcal G} = \bm{\mathcal F}$, respectively, with the condition $\bm{\mathcal Q} \geqslant 0$. Consequently, the objective function (\ref{of}) can be reformulated into an unconstrained format as follows:
\begin{equation}\label{objective function}
    \begin{aligned}
        &\min {\left\| \bm{\mathcal S} - \bm{\mathcal H}*\bm{\mathcal G}^{\mathrm{T}} \right \|}_F^2 + \lambda_1{\| \bm{\mathcal J}\|}_\Sp^p + \lambda_2{\| \bm{\mathcal F}\|}_\Sp^p \\
        & + \frac{\mu}{2} {\left\| \bm{\mathcal H}-\bm{\mathcal Q} + \frac{\bm{\mathcal Y_1}}{\mu}\right \|}_F^2 + \frac{\rho}{2} {\left\| \bm{\mathcal H}-\bm{\mathcal J} + \frac{\bm{\mathcal Y_2}}{\rho}\right \|}_F^2 \\
        & + \frac{\sigma}{2} {\left\| \bm{\mathcal G}-\bm{\mathcal F} + \frac{\bm{\mathcal Y_3}}{\sigma}\right \|}_F^2, \\
        & \emph{\textrm{s.t.}} \quad \bm{\mathcal Q}\geqslant 0, \bm{\mathcal H}^{\mathrm{T}}*\bm{\mathcal H} = \bm{\mathcal I}, \mathbf{G}^{(v)} \geqslant 0, \mathbf{G}^{(v)} \cdot \textbf{1} = \textbf{1}
    \end{aligned}
\end{equation}
Here, $\bm{\mathcal Y_1}$, $\bm{\mathcal Y_2}$, and $\bm{\mathcal Y_3}$ denote the Lagrange multipliers, while $\mu$, $\rho$, and $\sigma$ represent the penalty parameters.

To address this optimization challenge (\ref{objective function}), we employ an alternating optimization approach. Each variable is optimized independently while keeping the others fixed. The optimization process comprises five steps:

$\bullet$\textbf{Solving for $\bm{\mathcal G}$ with fixed $\bm{\mathcal Q}, \bm{\mathcal H}, \bm{\mathcal J}$, and $\bm{\mathcal F}$.} The objective function (\ref{objective function}) transforms into:
\begin{equation}\label{solveG1}
    \begin{aligned}
        \min_{\mathbf{G}^{(v)} \geqslant 0, \mathbf{G}^{(v)} \cdot \textbf{1} = \textbf{1}} {\left\| \bm{\mathcal S} - \bm{\mathcal H}*\bm{\mathcal G}^{\mathrm{T}} \right \|}_F^2 + \frac{\sigma}{2} {\left\| \bm{\mathcal G}-\bm{\mathcal F} + \frac{\bm{\mathcal Y_3}}{\sigma}\right \|}_F^2
    \end{aligned}
\end{equation}

Further, (\ref{solveG1}) can be reformulated in the frequency domain as:
\begin{equation}\label{solveG2}
    \begin{aligned}
        \min_{\mathbf{G}^{(v)} \geqslant 0, \mathbf{G}^{(v)} \cdot \textbf{1} = \textbf{1}}
        &\sum_{v=1}^{V}{\left\| \bm{\mathcal {\overline{S}}}^{(v)} - \bm{\mathcal {\overline{H}}}^{(v)} (\bm{\mathcal {\overline{G}}}^{(v)})^{\mathrm{T}} \right \|}_F^2 \\
        & + \sum_{v=1}^{V} \frac{\sigma}{2} {\left\| \bm{\mathcal {\overline{G}}}^{(v)} - \bm{\mathcal {\overline{F}}}^{(v)} + \frac{\bm{\mathcal {\overline{Y}}}_3^{(v)}}{\sigma}\right \|}_F^2,
    \end{aligned}
\end{equation}
where $\bm{\mathcal {\overline{G}}} = \mathrm{fft}({{\bm{\mathcal G}}},[\ ],3)$, and other variables are transformed similarly.

Clearly, the equation (\ref{solveG2}) can be reformulated as:
\begin{align}\label{solveG3}
   &\min_{\mathbf{G}^{(v)} \geqslant 0, \mathbf{G}^{(v)} \cdot \textbf{1} = \textbf{1}}
     -2\mathrm{tr}\left((\bm{\mathcal{\overline{G}}}^{(v)})^{\mathrm{T}} (\bm{\mathcal{\overline{S}}}^{(v)})^{\mathrm{T}} \bm{\mathcal{\overline{H}}}^{(v)}\right) \nonumber\\
    &- \sigma \mathrm{tr}\left((\bm{\mathcal{\overline{G}}}^{(v)})^{\mathrm{T}} \bm{\mathcal{\overline{W}}}_1^{(v)}\right) + (1+\frac{\sigma}{2}) \left( (\bm{\mathcal{\overline{G}}}^{(v)})^{\mathrm{T}} \bm{\mathcal{\overline{G}}}^{(v)} \right) ,
\end{align}
where $\bm{\mathcal{\overline{W}}}_1^{(v)} =  \bm{\mathcal{\overline{F}}}^{(v)} - \frac{\bm{\mathcal{\overline{Y}}}^{(v)}_3}{\sigma}$.

Further simplification of (\ref{solveG3}) yields:
\begin{equation}\label{solveG4}
    \begin{aligned}
        &\min (1+\frac{\sigma}{2}) \left[ \mathrm{tr} \left( (\bm{\mathcal{\overline{G}}}^{(v)})^{\mathrm{T}} \bm{\mathcal{\overline{G}}}^{(v)} \right) - 2\mathrm{tr}\left( (\bm{\mathcal{\overline{G}}}^{(v)})^{\mathrm{T}} \bm{\mathcal{\overline{B}}}_1^{(v)} \right)  \right]  \\
        & \quad \quad \emph{\textrm{s.t.}} \quad \mathbf{G}^{(v)} \geqslant 0, \mathbf{G}^{(v)} \cdot \textbf{1} = \textbf{1},
    \end{aligned}
\end{equation}
where $\bm{\mathcal{\overline{B}}}^{(v)}_1 = \left[ (\bm{\mathcal{\overline{S}}}^{(v)})^{\mathrm{T}} \bm{\mathcal{\overline{H}}}^{(v)} + (\sigma / 2) \bm{\mathcal{\overline{W}}}_1^{(v)} \right] / (1+\sigma / 2)$.

By returning (\ref{solveG4}) to the time domain and simplifying it, we can get:
\begin{equation}\label{solveG}
    \min_{\mathbf{G}^{(v)} \geqslant 0, \mathbf{G}^{(v)} \cdot \textbf{1} = \textbf{1}}
    (1+\frac{\sigma}{2}) \left|| \bm{\mathcal{G}} - \bm{\mathcal{B}}_1 \right||_{F}^2
\end{equation}

We can derive the solution of (\ref{solveG}) as $\mathbf{G}^{(v)}_i = (\mathbf{B_1}^{(v)}_i+ \gamma\textbf{1})$ according to \cite{nie2016constrained}, where $\mathbf{G}^{(v)}_i$ represents pulling the matrix $\mathbf{G}^{(v)}$ into a vector and $\gamma$ is a Lagrange multiplier.

$\bullet$\textbf{Solve for $\bm{\mathcal H}$ with fixed $\bm{\mathcal Q}, \bm{\mathcal G}, \bm{\mathcal J}$ and $\bm{\mathcal F}$.} (\ref{objective function}) then transforms into:
\begin{equation}\label{solveH1}
    \begin{aligned}
        & \min_{\bm{\mathcal H}^{\mathrm{T}}*\bm{\mathcal H} = \bm{\mathcal I}} {\left\| \bm{\mathcal S} - \bm{\mathcal H}*\bm{\mathcal G}^{\mathrm{T}} \right \|}_F^2  \\
        & + \frac{\mu}{2} {\left\| \bm{\mathcal H}-\bm{\mathcal Q} + \frac{\bm{\mathcal Y}_1}{\mu}\right \|}_F^2 + \frac{\rho}{2} {\left\| \bm{\mathcal H}-\bm{\mathcal J} + \frac{\bm{\mathcal Y_2}}{\rho}\right \|}_F^2,
    \end{aligned}
\end{equation}

Further, (\ref{solveH1}) can be reformulated in the frequency domain as:
\begin{equation}\label{solveH2}
    \begin{aligned}
        \min_{(\bm{\mathcal {\overline{H}}}^{(v)})^{\mathrm{T}} \bm{\mathcal {\overline{H}}}^{(v)} = \mathbf I}
        &\sum_{v=1}^{V}{\left\| \bm{\mathcal {\overline{S}}}^{(v)} - \bm{\mathcal {\overline{H}}}^{(v)} (\bm{\mathcal {\overline{G}}}^{(v)})^{\mathrm{T}} \right \|}_F^2 \\
        & + \sum_{v=1}^{V} \frac{\mu}{2} {\left\| \bm{\mathcal {\overline{H}}}^{(v)} - \bm{\mathcal {\overline{Q}}}^{(v)} + \frac{\bm{\mathcal {\overline{Y}}}_1^{(v)}}{\mu}\right \|}_F^2  \\
        & + \sum_{v=1}^{V} \frac{\rho}{2} {\left\| \bm{\mathcal {\overline{H}}}^{(v)} - \bm{\mathcal {\overline{J}}}^{(v)} + \frac{\bm{\mathcal {\overline{Y}}}_2^{(v)}}{\rho}\right \|}_F^2,
    \end{aligned}
\end{equation}

Clearly, the equation (\ref{solveH2}) can be reformulated as:
\begin{align}\label{solveH3}
   &\min_{(\bm{\mathcal{\overline{H}}}^{(v)})^{\mathrm{T}} \bm{\mathcal{\overline{H}}}^{(v)} = \mathbf{I}}
   -2\mathrm{tr}\left(\bm{\mathcal{\overline{G}}}^{(v)} (\bm{\mathcal{\overline{H}}}^{(v)})^{\mathrm{T}} \bm{\mathcal{\overline{S}}}^{(v)} \right) \nonumber\\
   &- \mu \mathrm{tr}\left((\bm{\mathcal{\overline{H}}}^{(v)})^{\mathrm{T}} \bm{\mathcal{\overline{W}}}_2^{(v)}\right) - \rho \mathrm{tr}\left((\bm{\mathcal{\overline{H}}}^{(v)})^{\mathrm{T}} \bm{\mathcal{\overline{W}}}_3^{(v)}\right),
\end{align}
where $\bm{\mathcal{\overline{W}}}_2^{(v)} =  \bm{\mathcal{\overline{Q}}}^{(v)} - \frac{\bm{\mathcal{\overline{Y}}}^{(v)}_1}{\mu}$ and $\bm{\mathcal{\overline{W}}}_3^{(v)} =  \bm{\mathcal{\overline{J}}}^{(v)} - \frac{\bm{\mathcal{\overline{Y}}}^{(v)}_2}{\rho}$.

Further simplification of (\ref{solveH3}) yields:
\begin{equation}\label{solveH4}
    \max_{(\bm{\mathcal {\overline{H}}}^{(v)})^{\mathrm{T}} \bm{\mathcal {\overline{H}}}^{(v)} = \mathbf I}
    \mathrm{tr} \left( (\bm{\mathcal {\overline{H}}}^{(v)})^{\mathrm{T}} \bm{\mathcal {\overline{B}}}^{(v)}_2 \right)
\end{equation}
where $\bm{\mathcal {\overline{B}}}^{(v)}_2 = 2 \bm{\mathcal {\overline{S}}}^{(v)} \bm{\mathcal {\overline{G}}}^{(v)} + \mu \bm{\mathcal {\overline{W}}}_2^{(v)} + \rho \bm{\mathcal {\overline{W}}}_3^{(v)}$.

To address (\ref{solveH4}), we introduce the following theorem:
\begin{theorem}\label{theorem solveH1}
    Given two matrices $\mathbf{G}$ and $\mathbf{P}$, where $\mathbf{G} (\mathbf{G})^{\mathrm{T}}=\mathbf{I}$ and $\mathbf{P}$ has the singular value decomposition $\mathbf{P}=\mathbf{\Lambda} \mathbf{S}(\mathbf{V})^{\mathrm{T}}$, the optimal solution of
    \begin{equation}\label{theorem solveH2}
        \max_{\mathbf{G} (\mathbf{G})^{\mathrm{T}}=\mathbf{I}} \mathrm{tr}(\mathbf{G} \mathbf{P})
    \end{equation}
    is $\mathbf{G}^\ast=\mathbf{V}[\mathbf{I},\mathbf{0}](\mathbf{\Lambda})^{\mathrm{T}}$.
\end{theorem}

\begin{proof}
    Starting from the SVD $\mathbf{P} = \mathbf{\Lambda} \mathbf{S} (\mathbf{V})^{\mathrm{T}}$ and in conjunction with Theorem \ref{theorem solveH1}, it becomes clear that
    \begin{equation}\label{theorem solveH3}
        \begin{aligned}
            \mathrm{tr}(\mathbf{G} \mathbf{P}) &= \mathrm{tr}(\mathbf{G} \mathbf{\Lambda}^{(v)} \mathbf{S} (\mathbf{V})^{\mathrm{T}}) \\
                                              &= \mathrm{tr}(\mathbf{S} (\mathbf{V})^{\mathrm{T}} \mathbf{G} \mathbf{\Lambda}) \\
                                              &= \mathrm{tr}(\mathbf{S} \mathbf{H}) \\
                                              &= \sum_i s_{ii} h_{ii},
        \end{aligned}
    \end{equation}
    where $\mathbf{H} = (\mathbf{V})^{\mathrm{T}} \mathbf{G} \mathbf{\Lambda}$, and $s_{ii}$ and $h_{ii}$ are the $(i,i)$ elements of $\mathbf{S}$ and $\mathbf{H}$, respectively. It is straightforward to verify that $\mathbf{H} (\mathbf{H})^{\mathrm{T}} = \mathbf{I}$, where $\mathbf{I}$ denotes the identity matrix. Consequently, $-1 \leqslant h_{ii} \leqslant 1$ and $s_{ii} \geqslant 0$, leading to the conclusion:
    \begin{equation}\label{theorem solveH4}
        \mathrm{tr}(\mathbf{G} \mathbf{P}) = \sum_i s_{ii} h_{ii} \leqslant \sum_i s_{ii}.
    \end{equation}
    Equality is achieved when $\mathbf{H}$ is an identity matrix, and $\mathrm{tr}(\mathbf{G} \mathbf{P})$ attains its maximum when $\mathbf{H} = [\mathbf{I}, \mathbf{0}]$.
\end{proof}

Hence, we can derive the solution of (\ref{solveH4}) as:
\begin{equation}\label{solveH}
    \bm{\mathcal {\overline{H}}}^{(v)} = \bm{{\overline{\Lambda}}}^{(v)}_2 (\bm{{\overline{V}}}^{(v)}_2)^{\top},
\end{equation}
where $\bm{{\overline{\Lambda}}}^{(v)}_2$ and $\bm{{\overline{V}}}^{(v)}_2$ are determined through the singular value decomposition (SVD) of $\bm{\mathcal {\overline{B}}}^{(v)}_2$.

$\bullet$\textbf{Solve $\bm{\mathcal Q}$ with fixed $\bm{\mathcal G}, \bm{\mathcal H}, \bm{\mathcal J}$ and $\bm{\mathcal F}$.} (\ref{objective function}) becomes:
\begin{equation}\label{solveQ1}
	\min_{\bm{\mathcal Q}\geqslant 0} \frac{\mu}{2} {\left\| \bm{\mathcal H}-\bm{\mathcal Q} + \frac{\bm{\mathcal Y}_1}{\mu}\right \|}_F^2
\end{equation}

Transforming this into the frequency domain and rearranging the variables, we obtain:
\begin{equation}\label{solveQ2}
    \min_{\bm{\mathcal {\overline{Q}}}^{(v)} \geqslant 0} \sum_{v=1}^{V}{ \frac{\mu}{2} \left\| \bm{\mathcal {\overline{Q}}}^{(v)} -(\bm{\mathcal {\overline{H}}}^{(v)} + \frac{\bm{\mathcal {\overline{Y}}}^{(v)}_1}{\mu}) \right \|}_F^2
\end{equation}

As per \cite{yang2021fast}, the solution for Equation (\ref{solveQ2}) is given by:
\begin{equation}\label{solveQ}
    \bm{\mathcal {\overline{Q}}}^{(v)} = \left(\bm{\mathcal {\overline{H}}}^{(v)} + \frac{\bm{\mathcal {\overline{Y}}}^{(v)}_1}{\mu} \right)_+
\end{equation}

\textbullet\ \textbf{Solve for $\bm{\mathcal J}$ with fixed $\bm{\mathcal Q}, \bm{\mathcal H}, \bm{\mathcal G}$, and $\bm{\mathcal F}$.} The equation (\ref{objective function}) is reformulated as:
\begin{equation}\label{solveJ1}
    \min \lambda_1{\|\bm{\mathcal J}\|}_\Sp^p + \frac{\rho}{2} {\left\|\bm{\mathcal H} - \bm{\mathcal J} + \frac{\bm{\mathcal Y}_2}{\rho}\right\|}_F^2,
\end{equation}
where, after completing the square with respect to $\bm{\mathcal J}$, we obtain
\begin{equation}\label{solveJ2}
    \begin{aligned}
        \bm{\mathcal J}^* = \arg \min \frac{1}{2}\left\|{\bm{\mathcal H} + \frac{\bm{\mathcal Y}_2}{\rho} - \bm{\mathcal J}}\right\|_F^2 + \frac{\lambda_1}{\rho}{\|\bm{\mathcal J}\|}_\Sp^p,
    \end{aligned}
\end{equation}
yielding a closed-form solution as outlined in Lemma \ref{T2} \cite{gao2020enhanced}:

\begin{lemma}\label{T2}
    Given ${\mathcal Z} \in {\mathbb{R}}^{n_1 \times n_2 \times n_3}$ with a t-SVD ${\mathcal Z} = {\mathcal U} * {\mathcal S} * {{\mathcal V}^{\mathrm{T}}}$, the optimal solution for
    \begin{equation}\label{tensor-gaozx-2020}
        \begin{array}{l}
            \min_{\mathcal X} \frac{1}{2}\left\| {{\mathcal X} - {\mathcal Z}} \right\|_F^2 + \tau \left\| {\mathcal X} \right\|_{{\Sp}}^p
        \end{array}
    \end{equation}
    is ${\mathcal X}^* = {\Gamma _\tau }({\mathcal Z}) = {\mathcal U}*\mathrm{ifft}({P_\tau }(\overline {\mathcal Z} ))*{{\mathcal V}^{\mathrm{T}}}$, where ${P_\tau }(\overline {\mathcal Z} )$ is an f-diagonal 3rd-order tensor. Its diagonal elements can be determined using the GST algorithm described in \cite{gao2020enhanced}.
\end{lemma}

Thus, the solution for (\ref{solveJ2}) is expressed as:
\begin{equation}\label{solveJ}
    \bm{\mathcal J}^* = {\Gamma _{\frac{\lambda_1}{\rho}}} (\bm{\mathcal H} + \frac{\bm{\mathcal Y}_2}{\rho}).
\end{equation}

$\bullet$\textbf{Solve for $\bm{\mathcal F}$ with fixed $\bm{\mathcal Q}, \bm{\mathcal H}, \bm{\mathcal J}$ and $\bm{\mathcal G}$.}  The objective function (\ref{objective function}) simplifies to:
\begin{equation}\label{solveF1}
	\min \lambda_2{\| \bm{\mathcal F}\|}_\Sp^p + \frac{\sigma}{2} {\left\| \bm{\mathcal G} - \bm{\mathcal F} + \frac{\bm{\mathcal Y}_3}{\sigma}\right \|}_F^2,
\end{equation}
By completing the square with respect to $\bm{\mathcal F}$, we obtain:
\begin{equation}\label{solveF2}
	\begin{aligned}
		\bm{\mathcal F}^* = \arg \min \frac{1}{2}\left\|{\bm{\mathcal G} + \frac{\bm{\mathcal Y}_3}{\sigma} - \bm{\mathcal F}}\right\|_F^2 + \frac{\lambda_2}{\sigma}{\|\bm{\mathcal F}\|}_\Sp^p,
	\end{aligned}
\end{equation}

Applying Lemma \ref{T2}, the solution to (\ref{solveF2}) is given by:
\begin{equation}\label{solveF}
    \bm{\mathcal F}^* = {\Gamma _{\frac{\lambda_2}{\sigma}}} (\bm{\mathcal G} + \frac{\bm{\mathcal Y}_3}{\sigma}).
\end{equation}

\begin{algorithm}[htb]
\caption{Interpretable Multi-view Clustering based on Anchor Graph Tensor Factorization (AGTF)}
\label{A1}
\begin{algorithmic}[1] 
\REQUIRE Data matrices $\{{\mathbf{X}}^{(v)}\}_{v=1}^{V}\in \mathbb{R}^{N\times d_v}$; anchors numbers $m$; cluster number $K$.\\ 
\ENSURE Cluster labels $\mathbf{Y}$ of each data points.\\ 
\STATE \textbf{Initialize}: $\mu=10^{-5}$, $\rho=10^{-5}$, $\sigma=10^{-5}$, $\eta=1.3$, $\bm{\mathcal Y}_1=0$, $\bm{\mathcal Y}_2=0$, $\bm{\mathcal Y}_3=0$ and $\mathbf{\overline {\bm{\mathcal H}}}^{(v)}$ is identity matrix;
\STATE Compute graph matrix $\mathbf S^{(v)}$ of each views;
\WHILE{not condition}
\STATE Update $\bm{\mathcal {{G}}}$ by solving  (\ref{solveG});
\STATE Update $\bm{\mathcal {\overline{H}}}^{(v)}$ by solving  (\ref{solveH});
\STATE Update $\bm{\mathcal {\overline{Q}}}^{(v)}$ by solving  (\ref{solveQ});
\STATE Update ${\bm{{\mathcal J}}}$ by using  (\ref{solveJ});
\STATE Update ${\bm{{\mathcal F}}}$ by using  (\ref{solveF});
\STATE Update $\bm{\mathcal Y}_1$, $\bm{\mathcal Y}_2$, $\bm{\mathcal Y}_3$, $\mu$, $\rho$ and $\sigma$: $\bm{\mathcal Y}_1=\bm{\mathcal Y}_1+\mu(\bm{\mathcal H}-\bm{\mathcal Q})$, $\bm{\mathcal Y}_2=\bm{\mathcal Y}_2+\rho(\bm{\mathcal H}-\bm{\mathcal J})$, $\bm{\mathcal Y}_3=\bm{\mathcal Y}_3+\sigma(\bm{\mathcal G}-\bm{\mathcal F})$, $\mu=\min(\eta\mu, 10^{13})$, $\rho=\min(\eta\rho, 10^{13})$, $\sigma=\min(\eta\sigma, 10^{13})$;
\ENDWHILE
\STATE Calculate the $K$ clusters by using \\
$\mathbf Q=\sum_{v=1}^V \mathbf Q^{(v)} / V$;
\STATE \textbf{return} Clustering result.
\end{algorithmic}
\end{algorithm}

%

\begin{table}[t]
\caption{Multi-view datasets used in our experiments}
\label{datasets}
\centering
\scalebox{0.60}
{
\begin{tabular}{ccccc}
\toprule[2pt]
\#Dataset 	& \#Samples & \#View & \#Class 	& \multicolumn{1}{l}{\#Feature} \\
\midrule
MSRC    & 210  & 5 & 7 & \multicolumn{1}{l}{24, 576, 512, 256, 254} \\
HandWritten4   & 2000 & 4 & 10 & \multicolumn{1}{l}{76, 216, 47, 6} \\
Mnist4	& 4000 	& 3	& 4	& \multicolumn{1}{l}{30, 9, 30}   \\
Scene15 & 4485 	& 3	& 15	& \multicolumn{1}{l}{1800, 1180, 1240}   \\
Reuters & 18758 & 5 & 6 & \multicolumn{1}{l}{21531, 24892, 34251, 15506, 11547} \\
Noisy MNIST	& 50000 	& 2	& 10	& \multicolumn{1}{l}{784, 784}   \\
\toprule[2pt]
\end{tabular}
}
\end{table}
\subsection{Convergence Analysis}
\begin{theorem}\label{thm1}[Convergence Analysis of Algorithm~\ref{A1}]
Let $\mathcal{P}_{k}=\{\bm{{\mathcal{Q}}}_{k},\bm{{\mathcal{H}}}_{k}, \bm{{\mathcal{G}}}_{k}, \bm{{\mathcal{J}}}_{k}, \bm{{\mathcal{F}}}_{k},\bm{{\mathcal{Y}}}_{1,k}, \bm{{\mathcal{Y}}}_{2,k}, \bm{{\mathcal{Y}}}_{3,k}\},\ 1\leq k< \infty$ in \eqref{objective function} be a sequence generated
by \textbf{Algorithm~1}, then
\begin{enumerate}
\item  $\mathcal{P}_{k}$ is bounded;
\item  Any accumulation point of $\mathcal{P}_{k}$ is a stationary KKT point of \eqref{objective function}.
\end{enumerate}
\end{theorem}
The proof will be provided in the appendix and we need to mention that the KKT conditions can be used to determine the stop conditions for Algorithm \ref{A1}, which are
$\| \bm{{\mathcal{Q}}}_{k}-\bm{{\mathcal{H}}}_{k} \|_\infty\leq\varepsilon$, $\| \bm{{\mathcal{Q}}}_{k}-\bm{{\mathcal{J}}}_{k} \|_\infty\leq\varepsilon$, $\| \bm{{\mathcal{G}}}_{k}-\bm{{\mathcal{F}}}_{k} \|_\infty\leq\varepsilon$.

\section{Experiments}
In this section, we carry out comprehensive comparative experiments on our proposed methods. The efficacy of clustering is gauged using three key metrics: Accuracy (ACC), Normalized Mutual Information (NMI), and Purity. Higher values of these metrics indicate superior clustering performance. To ensure reliability, we conducted 20 independent trials for each method, recording the mean and variance of the results. Further details, including experimental setups and hyper-parameter choices, are provided in the appendix.

The datasets employed in our experiments are detailed in Table \ref{datasets}.
We compare our approach against the following state-of-the-art methods: CSMSC \cite{luo2018consistent}, GMC \cite{wang2019gmc}, ETLMSC \cite{WuLZ19}, LMVSC \cite{kang2020large}, FMCNOF \cite{yang2021fast}, SFMC \cite{li2020multiview}, FPMVS-CAG \cite{wang2021fast}, and Orth-NTF \cite{li2023orthogonal}.

\begin{table*}[h]
\caption{Clustering performance on MSRC, HandWritten4, Mnist4 and Scene15.}
\label{result1}
\centering
\resizebox{\textwidth}{!}
{
\begin{tabular}{c|ccc|ccc|ccc|ccc}
\toprule[2pt]
Datasets    &\multicolumn{3}{c}{MSRC}  &\multicolumn{3}{c}{HandWritten4} &\multicolumn{3}{c}{Mnist4}  &\multicolumn{3}{c}{Scene15}\\
\midrule
Metrices  &ACC &NMI &Purity  &ACC &NMI &Purity   &ACC &NMI &Purity   &ACC &NMI &Purity\\
\midrule[1pt]
CSMSC   &0.758$\pm$0.007 &0.735$\pm$0.010 &0.793$\pm$0.008      &0.806$\pm$0.001 &0.793$\pm$0.001 &0.867$\pm$0.001  &0.641$\pm$0.000 &0.601$\pm$0.010 &0.728$\pm$0.008    &0.334$\pm$0.008 &0.313$\pm$0.005 &0.378$\pm$0.003 \\
GMC     &0.895$\pm$0.000 &0.809$\pm$0.000 &0.895$\pm$0.000      &0.861$\pm$0.000 &0.859$\pm$0.000 &0.861$\pm$0.000  &0.920$\pm$0.000 &0.807$\pm$0.000 &0.920$\pm$0.000    &0.140$\pm$0.000 &0.058$\pm$0.000 &0.146$\pm$0.000 \\
ETLMSC  &0.962$\pm$0.000 &0.937$\pm$0.000 &0.962$\pm$0.000      &0.938$\pm$0.001 &0.893$\pm$0.001 &0.938$\pm$0.001  &0.934$\pm$0.000 &0.847$\pm$0.000 &0.934$\pm$0.000    &0.709$\pm$0.000  &0.774$\pm$0.000 &\underline{0.887$\pm$0.000} \\
LMVSC   &0.814$\pm$0.000 &0.717$\pm$0.000 &0.814$\pm$0.000      &0.904$\pm$0.000 &0.831$\pm$0.000 &0.904$\pm$0.000  &0.892$\pm$0.000 &0.726$\pm$0.000 &0.892$\pm$0.000    &0.355$\pm$0.000 &0.331$\pm$0.000 &0.399$\pm$0.000 \\
FMCNOF  &0.440$\pm$0.039 &0.345$\pm$0.046 &0.449$\pm$0.042      &0.385$\pm$0.092 &0.370$\pm$0.092 &0.386$\pm$0.090  &0.697$\pm$0.119 &0.490$\pm$0.102 &0.711$\pm$0.096    &0.218$\pm$0.033 &0.166$\pm$0.022 &0.221$\pm$0.029 \\
SFMC    &0.810$\pm$0.000 &0.721$\pm$0.000 &0.810$\pm$0.000      &0.853$\pm$0.000 &0.871$\pm$0.000 &0.873$\pm$0.000  &0.916$\pm$0.000 &0.797$\pm$0.000 &0.916$\pm$0.000    &0.188$\pm$0.000 &0.135$\pm$0.000 &0.202$\pm$0.000 \\
FPMVS-CAG &0.786$\pm$0.000 &0.686$\pm$0.000 &0.786$\pm$0.000    &0.744$\pm$0.000 &0.753$\pm$0.000 &0.744$\pm$0.000  &0.885$\pm$0.000 &0.715$\pm$0.000 &0.885$\pm$0.000    &0.463$\pm$0.000 &0.486$\pm$0.000 &0.481$\pm$0.000 \\
Orth-NTF &\underline{0.990$\pm$0.000} &\underline{0.978$\pm$0.000}    &\underline{0.990$\pm$0.000}  &\underline{0.985$\pm$0.000} &\underline{0.969$\pm$0.000} &\underline{0.985$\pm$0.000}   &\underline{0.977$\pm$0.000} &\underline{0.926$\pm$0.000} &\underline{0.977$\pm$0.000} &\underline{0.758$\pm$0.000} &\underline{0.804$\pm$0.000} &0.759$\pm$0.000 \\
Ours    &\textbf{1.000$\pm$0.000} &\textbf{1.000$\pm$0.000} &\textbf{1.000$\pm$0.000}  &\textbf{0.995$\pm$0.000} &\textbf{0.989$\pm$0.000} &\textbf{0.995$\pm$0.000}   &\textbf{0.996$\pm$0.000} &\textbf{0.983$\pm$0.000} &\textbf{0.996$\pm$0.000}  &\textbf{0.870$\pm$0.000} &\textbf{0.906$\pm$0.000} &\textbf{0.902$\pm$0.000} \\
\midrule[2pt]
\end{tabular}
}
\end{table*}

\begin{table}[h]
\caption{Clustering performance on Reuters and NoisyMnist ("OM" means out of memory, and "-" means the algorithm ran for more than three hours.)}
\label{result2}
\centering
\resizebox{\columnwidth}{!}
{
\begin{tabular}{c|ccc|ccc}
\toprule[2pt]
Datasets    &\multicolumn{3}{c}{Reuters} &\multicolumn{3}{c}{NoisyMnist}\\
\midrule
Metrices   & ACC & NMI & Purity        & ACC & NMI & Purity\\
\midrule[1pt]
CSMSC    &OM &OM &OM                 &OM &OM &OM \\
GMC     &- &- &-                    &- &- &- \\
ETLMSC  &OM &OM &OM                 &OM &OM &OM \\
LMVSC   &0.589$\pm$0.000 &0.335$\pm$0.000 &0.615$\pm$0.000        &0.388$\pm$0.000 &0.344$\pm$0.000 &0.434$\pm$0.000 \\
FMCNOF   &0.343$\pm$0.000 &0.125$\pm$0.000 &0.358$\pm$0.000        &0.333$\pm$0.000 &0.237$\pm$0.000 &0.340$\pm$0.000 \\
SFMC     &0.602$\pm$0.000 &0.354$\pm$0.000 &0.604$\pm$0.000        &0.699$\pm$0.000 &0.681$\pm$0.000 &0.727$\pm$0.000 \\
FPMVS-CAG  &0.526$\pm$0.000 &0.323$\pm$0.000 &0.603$\pm$0.000       &0.554$\pm$0.000 &0.513$\pm$0.000 &0.567$\pm$0.000 \\
Orth-NTF  &\underline{0.694$\pm$0.000} &\underline{0.686$\pm$0.000} &\textbf{0.809$\pm$0.000}  &\underline{0.701$\pm$0.000} &\underline{0.729$\pm$0.000} &\underline{0.747$\pm$0.000} \\
Ours    &\textbf{0.796$\pm$0.000} &\textbf{0.731$\pm$0.000} &\underline{0.796$\pm$0.000} &\textbf{0.819$\pm$0.000} &\textbf{0.837$\pm$0.000} &\textbf{0.851$\pm$0.000} \\
\midrule[2pt]
\end{tabular}
}
\end{table}

\subsection{Clustering Performance}
Table \ref{result1} presents the outcomes of comparative experiments conducted on four small to medium-sized datasets, while Table \ref{result2} displays these outcomes for two large-scale datasets. The superior result is highlighted in \textbf{bold}, with the second-best result being \underline{underlined}. It is evident that the methodology introduced in this study outperforms all other compared methods in the context of the designated metrics.

Compared with the non-negative matrix factorization of the fused anchor graph, our approach constructs the anchor graph of different views into tensors and performs tensor decomposition of the anchor graph tensors. This improvement stems from considering both the complementary and spatial structural information embedded in different views' anchor graphs.

Building upon Orth-NTF, our method further refines clustering performance. We introduce an additional tensor derived from NTF, designated as the labeling tensor, which delineates the relationships between anchors and clusters. Imposing tensor low-rank constraints on this tensor further augments the clustering accuracy.

\begin{figure}[htbp]
	\centering
    \setlength{\abovecaptionskip}{-0.1cm}
	\includegraphics[width=1.0\linewidth]{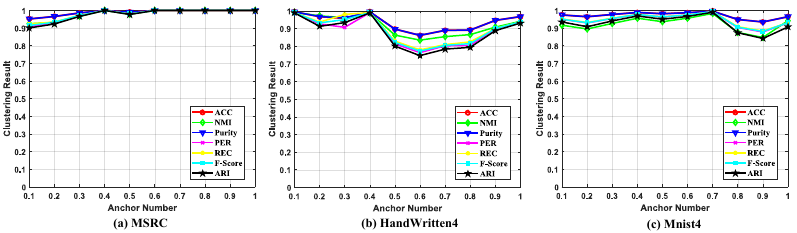}
	\caption{Clustering performance with different anchor rate on MSRC, HandWritten4 and Mnist4.}
	\label{anchorfig}
    \vspace{-3mm}
\end{figure}

\begin{figure}[htbp]
	\centering
    \setlength{\abovecaptionskip}{-0.1cm}
	\includegraphics[width=1.0\linewidth]{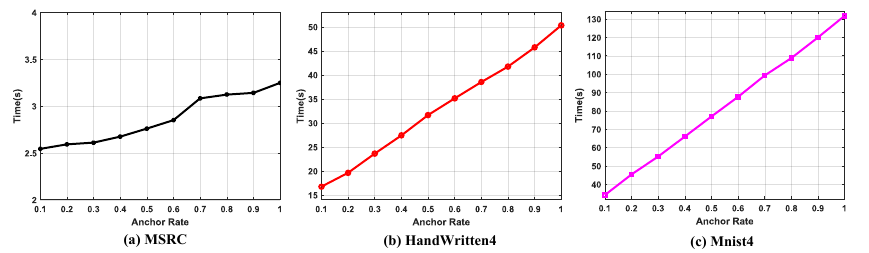}
	\caption{Running time (sec.) with different number of anchors on MSRC, HandWritten4 and Mnist4.}
	\label{timefig}
    \vspace{-5mm}
\end{figure}

\begin{figure}[htbp]
	\centering
    \setlength{\abovecaptionskip}{-0.1cm}
	\includegraphics[width=1.0\linewidth]{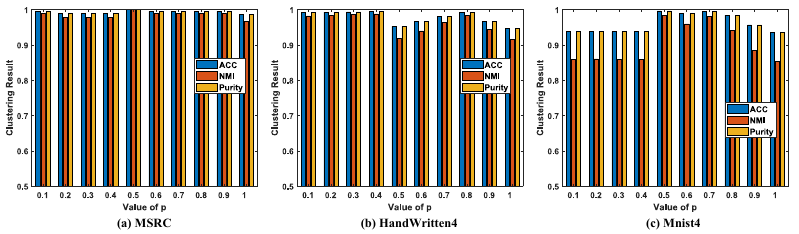}
	\caption{The influence of $p$ on clustering performance on MSRC, HandWritten4 and Mnist4.}
	\label{pfig}
    \vspace{0mm}
\end{figure}

\begin{figure}[htbp]
	\centering
    \setlength{\abovecaptionskip}{-0.1cm}
	\includegraphics[width=1.0\linewidth]{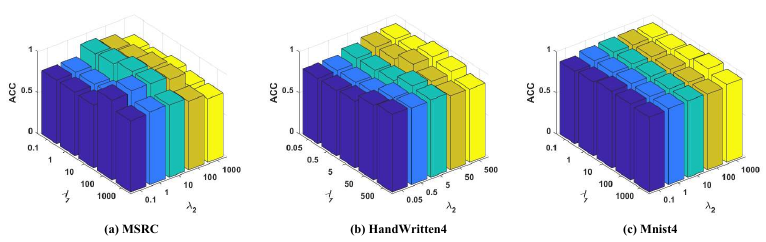}
	\caption{The influence of $\lambda_1$ and $\lambda_2$ on clustering performance on MSRC, HandWritten4 and Mnist4.}
	\label{lambdafig}
    \vspace{0mm}
\end{figure}

\subsection{Parametric Analysis}
In this subsection, we examine the impact of various parameters—anchor rate, the value of $p$, and the values of $\lambda_1$ and $\lambda_2$—on clustering performance.

Figure \ref{anchorfig} demonstrates how the anchor rate influences the seven clustering metrics. We observe that an anchor rate around 0.4 yields satisfactory clustering results across all three datasets. The MSRC dataset, being smaller, requires a higher anchor rate to adequately represent its limited sample size, thereby improving clustering performance. Conversely, medium-sized datasets like HandWritten4 and Mnist4, achieve relatively excellent clustering results at anchor rates of 0.4 and 0.7, respectively. Additionally, as depicted in Figure \ref{timefig}, the algorithm's running time is approximately linear with respect to the anchor rate. Hence, the lowest feasible anchor rate is recommended, provided it satisfies clustering performance criteria.

Figure \ref{pfig} explores the influence of the value of $p$ in the tensor Schatten $p$-norm on clustering outcomes. Notably, our model uses the same $p$ value for both tensor Schatten $p$-norm regularization terms to simplify parameter tuning. Optimal clustering is achieved on the MSRC, HandWritten4, and Mnist4 datasets with $p$ values of 0.5, 0.4, and 0.5, respectively. This finding suggests that the tensor Schatten $p$-norm effectively leverages the low-rank nature of views, enhancing the extraction of complementary information across different views and thus improving clustering performance.

Figure \ref{lambdafig} presents the clustering performance under various $\lambda_1$ and $\lambda_2$ combinations. To determine the optimal values of $\lambda_1$ and $\lambda_2$, we first estimate their range based on the magnitude of the tensor Schatten $p$-norm regularization, then refine our selection through detailed fine-tuning within this range.

\subsection{Ablation Experiments}
As shown in Table \ref{ablation}, we conducted ablation experiments on two tensor Schatten $p$-norm regularization on MSRC and Scene15. It can be seen that the tensor Schatten $p$-norm of the cluster label on the samples is more important than that on the anchors, i.e., ${\left \|{\bm{\mathcal H}}\right \|}_\Sp^p$ is more important than ${\left \|{\bm{\mathcal G}}\right \|}_\Sp^p$.
The rationale behind this approach is that imposing low-rank tensor constraints on $\bm{\mathcal H}$ effectively excavates both the complementary and spatial structural information inherent in the clustering index matrices derived from various views. Meanwhile, $\bm{\mathcal G}$ encapsulates the clustering distribution of the anchors. The concurrent application of constraints on both $\bm{\mathcal H}$ and $\bm{\mathcal G}$ significantly enhances the overall clustering efficacy.

\begin{table}[h]
\caption{Ablation experiments on MSRC and Scene15}
\label{ablation}
\centering
\resizebox{\columnwidth}{!}
{
\begin{tabular}{cc|ccc|ccc}
\midrule[2pt]
\multicolumn{2}{c}{Situations}	  &\multicolumn{3}{c}{MSRC} &\multicolumn{3}{c}{Scene15} \\
\midrule
${\left \|{\bm{\mathcal H}}\right \|}_\Sp^p$ &${\left \|{\bm{\mathcal G}}\right \|}_\Sp^p$ & ACC & NMI & Purity & ACC & NMI & Purity \\
\midrule[1pt]
\ding{56}  &\ding{56} &0.789 &0.713 &0.789 &0.499 &0.490 &0.541 \\
\ding{56}  &\ding{52} &0.952 &0.908 &0.952 &0.628 &0.562 &0.649 \\
\ding{52}  &\ding{56} &0.971 &0.943 &0.971 &0.837 &0.860 &0.858 \\
\ding{52}  &\ding{52} &1.000 &1.000 &1.000 &0.870 &0.906 &0.902 \\
\bottomrule[2pt]
\end{tabular}
}
\end{table}

\section{Conclusion}\label{conclusion}
In this paper, we execute non-negative tensor factorization on anchor graph tensors, which are assembled from different view-specific anchor graphs. This simultaneous processing of all view data guarantees the comprehensive utilization of inter-view information. Moreover, we assert that the tensors derived from this decomposition reflect the clustering of both samples and anchors, thereby enhancing the interpretability of the decomposition. We introduce an optimization approach specifically tailored for non-negative tensor factorization and affirm the convergence of our algorithm. Ultimately, extensive experimental evaluations substantiate the efficacy of our proposed methodology.


{\small
\bibliographystyle{IEEEtran}
\bibliography{egbib}
}

\onecolumn
\appendix
\section{Appendix: Proof of Theorem \ref{thm1}}
\subsection{Proof of the 1st part}
\begin{lemma}[Proposition 6.2 of \cite{lewis2005nonsmooth}]\label{lewis}
Suppose $F: \mathbb{R}^{n_1\times n_2}\rightarrow \mathbb{R}$ is represented as $F(X)=f \circ \sigma(X)$, where $X\in\mathbb{R}^{n_1\times n_2} $ with SVD
 $X=U \mathrm{diag}(\sigma_1, \ldots, \sigma_n) V^{\mathrm{T}}$, $n=\min(n_1, n_2)$, and $f$ is differentiable. The gradient of $F(X)$ at $X$ is
\begin{equation}
\label{deritheorem}
\frac{\partial F(X)}{\partial X}=U \mathrm{diag}(\theta) V^{\mathrm{T}},
\end{equation}
where $\theta=\frac{\partial f(y)}{\partial y}|_{y=\sigma (X)}$.
\end{lemma}

To minimize $\mathcal{J}$ at the $(k+1)$\textsuperscript{th} step in \eqref{solveJ2}, the optimal solution $\mathcal{J}_{k+1}$ must satisfy the first-order optimality condition:
\begin{equation}
    \lambda\nabla_{\bm{\mathcal {J}}}\|\bm{\mathcal {J}}_{k+1}\|^p_{\Sp}+\rho_k(\bm{\mathcal {J}}_{k+1}-\bm{{\mathcal{H}}}_{k+1}-\dfrac{1}{\rho_k}\bm{{\mathcal{Y}}}_{2,k})=0.
\end{equation}

Recall that for $0<p<1$, to address the singularity of $(|\eta|^p)'=p\eta/|\eta|^{2-p}$ near $\eta=0$, we adopt the approximation for $0<\epsilon\ll 1$:
\begin{equation}
    \partial |\eta|^p\approx\dfrac{p\eta}{\max\{\epsilon^{2-p},|\eta|^{2-p}\}}.
\end{equation}
Defining $\overline {\bm{\mathcal {J}}}^{(i)}={\overline {\bm{\mathcal {U}}}}^{(i)}\mathrm{diag}\left(\sigma_j(\overline {\bm{\mathcal {J}}}^{(i)})\right){\overline {\bm{\mathcal {V}}}}^{(i)\mathrm{H}}$, it follows from Lemma~\ref{lewis} that
\begin{align}
    \frac{\partial \|{\overline {\bm{\mathcal {J}}}}^{(i)}\|^p_{\Sp}}{\partial{\overline {\bm{\mathcal {J}}}}^{(i)}}={\overline {\bm{\mathcal {U}}}}^{(i)}\mathrm{diag}\left(\dfrac{p\sigma_j(\overline {\bm{\mathcal {J}}}^{(i)})}{\max\{\epsilon^{2-p},|\sigma_j(\overline {\bm{\mathcal {J}}}^{(i)})|^{2-p}\}}\right){\overline {\bm{\mathcal {V}}}}^{(i)\mathrm{H}}.
\end{align}
Hence, we obtain
\begin{align}
    \dfrac{p\sigma_j(\overline {\bm{\mathcal {J}}}^{(i)})}{\max\{\epsilon^{2-p},|\sigma_j(\overline {\bm{\mathcal {J}}}^{(i)})|^{2-p}\}}\leq \dfrac{p}{\epsilon^{1-p}} \Longrightarrow \left\|\frac{\partial \|{\overline {\bm{\mathcal {J}}}}^{(i)}\|^p_{\Sp}}{\partial{\overline {\bm{\mathcal {J}}}}^{(i)}}\right\|^2_F\leq \sum^{n}_{i=1} \dfrac{p^2}{\epsilon^{2(1-p)}}.
\end{align}
Thus, $\frac{\partial \|{\overline {\bm{\mathcal {J}}}}\|^p_{\Sp}}{\partial{\overline {\bm{\mathcal {J}}}}}$ is bounded.

Let us define $\widetilde{\mathbf{F}}_{V} = \frac{1}{\sqrt{V}}\mathbf{F}_{V}$, where $\mathbf{F}_{V}$ represents the discrete Fourier transform matrix of size $V\times V$, and $\mathbf{F}^{\mathrm{H}}_{V}$ denotes its conjugate transpose. Considering $\bm{\mathcal {J}}=\overline {\bm{\mathcal {J}}}\times_3 \widetilde{\mathbf{F}}_{V}$ and applying the chain rule in matrix calculus, we obtain that
$$\nabla_{\bm{\mathcal {J}}}\|\bm{\mathcal {J}}\|^p_{\Sp}=\frac{\partial \|{\bm{\mathcal {J}}}\|^p_{\Sp}}{\partial{\overline {\bm{\mathcal {J}}}}}\times_3 \widetilde{\mathbf{F}}_{V}^{\mathrm{H}}$$
is bounded.

Consequently, we have
\begin{align*}
&\bm{{\mathcal{Y}}}_{2,k+1}=\bm{{\mathcal{Y}}}_{2,k}+\rho_{k}(\bm{{\mathcal{H}}}_{k+1}-\bm{\mathcal {J}}_{k+1})\\&\Longrightarrow \lambda_1\nabla_{\bm{\mathcal {J}}}\|\bm{\mathcal {J}}_{k+1}\|^p_{\Sp}=\bm{{\mathcal{Y}}}_{2,k+1},
\end{align*}
indicating that $\bm{{\mathcal{Y}}}_{2,k+1}$ is also bounded. Similarly, based on \eqref{solveJ2} and the update $\bm{{\mathcal{Y}}}_{3,k+1}=\bm{{\mathcal{Y}}}_{3,k}+\sigma_{k}(\bm{{\mathcal{G}}}_{k+1}-\bm{\mathcal {F}}_{k+1})$, it follows that $\bm{{\mathcal{Y}}}_{3,k+1}$ is bounded as well.

To minimize ${\bm{\mathcal{\bar H}}^{(v)}}$ at step $k+1$ in \eqref{solveH2}, the optimal ${\bm{\mathcal{\bar H}}_{k+1}^{(v)}}$ must satisfy the first-order optimality condition
$$2{\bm{\mathcal{\bar H}}_{k+1}^{(v)}}={\bm{\mathcal{\bar Q}}_{k}^{(v)}}+\frac{\bm{\mathcal{\bar{Y}}}_{1,k}}{\mu_k}+{\bm{\mathcal{\bar J}}_{k}^{(v)}}+\frac{\bm{\mathcal{\bar{Y}}}_{2,k}}{\rho_k}.$$
Utilizing the update rule $\bm{\mathcal{\bar Y}}_{1,k+1} ^{(v)}= {\bm{\mathcal{\bar{Y}}}_{1,k}^{(v)}} + \mu_k({\bm{\mathcal{\bar Q}}_{k}^{(v)}}-{\bm{\mathcal{\bar H}}_{k}^{(v)}})$ and $\bm{\mathcal{\bar Y}}_{2,k+1} ^{(v)}= {\bm{\mathcal{\bar{Y}}}_{2,k}^{(v)}} + \rho_k({\bm{\mathcal{\bar J}}_{k}^{(v)}}-{\bm{\mathcal{\bar H}}_{k}^{(v)}})$, we derive that
$$\frac{\bm{\mathcal{\bar Y}}_{1,k+1}^{(v)}}{\mu_k}+\frac{\bm{\mathcal{\bar Y}}_{2,k+1} ^{(v)}}{\rho_k}+2({\bm{\mathcal{\bar H}}_{k}^{(v)}}-{\bm{\mathcal{\bar H}}_{k+1}^{(v)}})=0.$$
Therefore, due to the boundedness of $\bm{{\mathcal{Y}}}_{2,k+1}$, we have $\bm{{\mathcal{Y}}}_{1,k+1}$ is also bounded.

Furthermore, utilizing the update rules $\bm{\mathcal Y}_1=\bm{\mathcal Y}_1+\mu(\bm{\mathcal H}-\bm{\mathcal Q})$, $\bm{\mathcal Y}_2=\bm{\mathcal Y}_2+\rho(\bm{\mathcal H}-\bm{\mathcal J})$, and $\bm{\mathcal Y}_3=\bm{\mathcal Y}_3+\sigma(\bm{\mathcal G}-\bm{\mathcal F})$, we can infer (for $i=1,2,3$)
\begin{align}
\label{eq:Lk_ieq}
&\bm{\mathcal L} (\bm{\mathcal Q}_{k+1}, \bm{\mathcal G}_{k+1}, \bm{\mathcal H}_{k+1},\bm{\mathcal J}_{k+1},\bm{\mathcal F}_{k+1};\bm{{\mathcal{Y}}}_{i,k})\\
& \leq  \bm{\mathcal L} (\bm{\mathcal Q}_{k}, \bm{\mathcal G}_{k}, \bm{\mathcal H}_{k},\bm{\mathcal J}_{k},\bm{\mathcal F}_{k};\bm{{\mathcal{Y}}}_{i,k}) \nonumber\\
&= \bm{\mathcal L} (\bm{\mathcal Q}_{k}, \bm{\mathcal G}_{k}, \bm{\mathcal H}_{k},\bm{\mathcal J}_{k},\bm{\mathcal F}_{k};\bm{{\mathcal{Y}}}_{i,k-1})\nonumber\\
&+\frac{\sigma_k+\sigma_{k-1}}{2\sigma^2_{k-1}}\|\bm{{\mathcal{Y}}}_{3,k}-\bm{{\mathcal{Y}}}_{3,k-1}\|_F^2+ \frac{\|\bm{{\mathcal{Y}}}_{3,k}\|_F^2}{2\sigma_k}- \frac{\|\bm{{\mathcal{Y}}}_{3,k-1}\|_F^2}{2\sigma_{k-1}}\nonumber\\
&+\frac{\rho_k+\rho_{k-1}}{2\rho^2_{k-1}}\|\bm{{\mathcal{Y}}}_{2,k}-\bm{{\mathcal{Y}}}_{2,k-1}\|_F^2+ \frac{\|\bm{{\mathcal{Y}}}_{2,k}\|_F^2}{2\rho_k}- \frac{\|\bm{{\mathcal{Y}}}_{2,k-1}\|_F^2}{2\rho_{k-1}}\nonumber\\
&+\frac{\mu_k+\mu_{k-1}}{2\mu^2_{k-1}}\|\bm{{\mathcal{Y}}}_{1,k}-\bm{{\mathcal{Y}}}_{1,k-1}\|_F^2+ \frac{\|\bm{{\mathcal{Y}}}_{1,k}\|_F^2}{2\mu_k}- \frac{\|\bm{{\mathcal{Y}}}_{1,k-1}\|_F^2}{2\mu_{k-1}}.\nonumber
\end{align}
Therefore, by summing both sides of \eqref{eq:Lk_ieq} from $k=1$ to $n$, we obtain
\begin{equation}
\begin{aligned}
&\bm{\mathcal L} (\bm{\mathcal Q}_{n+1}, \bm{\mathcal G}_{n+1}, \bm{\mathcal H}_{n+1},\bm{\mathcal J}_{n+1},\bm{\mathcal F}_{n+1};\bm{{\mathcal{Y}}}_{i,n}) \\
\leq & \bm{\mathcal L} (\bm{\mathcal Q}_{1}, \bm{\mathcal G}_{1}, \bm{\mathcal H}_{1},\bm{\mathcal J}_{1},\bm{\mathcal F}_{1};\bm{{\mathcal{Y}}}_{i,0})) \\
+&\frac{\|\bm{{\mathcal{Y}}}_{3,n}\|_F^2}{2\sigma_n}- \frac{\|\bm{{\mathcal{Y}}}_{3,0}\|_F^2}{2\sigma_{0}}+\sum_{k=1}^n\left(\frac{\sigma_k+\sigma_{k-1}}{2\sigma^2_{k-1}}\|\bm{{\mathcal{Y}}}_{3,k}-\bm{{\mathcal{Y}}}_{3,k-1}\|_F^2\right) \\
+&\frac{\|\bm{{\mathcal{Y}}}_{2,n}\|_F^2}{2\rho_n}- \frac{\|\bm{{\mathcal{Y}}}_{2,0}\|_F^2}{2\rho_{0}}+\sum_{k=1}^n\left(\frac{\rho_k+\rho_{k-1}}{2\rho^2_{k-1}}\|\bm{{\mathcal{Y}}}_{2,k}-\bm{{\mathcal{Y}}}_{2,k-1}\|_F^2\right) \\
+&\frac{\|\bm{{\mathcal{Y}}}_{1,n}\|_F^2}{2\mu_n}- \frac{\|\bm{{\mathcal{Y}}}_{1,0}\|_F^2}{2\mu_{0}}+\sum_{k=1}^n\left(\frac{\mu_k+\mu_{k-1}}{2\mu^2_{k-1}}\|\bm{{\mathcal{Y}}}_{1,k}-\bm{{\mathcal{Y}}}_{1,k-1}\|_F^2\right).
\label{eq:Lk_sum}
\end{aligned}
\end{equation}

Consider that
\[
\sum_{k=1}^{\infty}\frac{\rho_k+\rho_{k-1}}{2\rho_{k-1}^2}<\infty, \quad \sum_{k=1}^{\infty}\frac{\mu_k+\mu_{k-1}}{2\mu_{k-1}^2}<\infty, \quad \sum_{k=1}^{\infty}\frac{\sigma_k+\sigma_{k-1}}{2\sigma_{k-1}^2}<\infty,
\]
it follows that the right-hand side of \eqref{eq:Lk_sum} is finite, thereby ensuring that $\bm{\mathcal L} (\bm{\mathcal Q}_{n+1}, \bm{\mathcal G}_{n+1}, \bm{\mathcal H}_{n+1}, \bm{\mathcal J}_{n+1}, \bm{\mathcal F}_{n+1};\bm{{\mathcal{Y}}}_{i,n})$ is bounded. Referring to \eqref{objective function},
\begin{align}
\label{eq:Ln_bdd}
&\bm{\mathcal L} (\bm{\mathcal Q}_{n+1}, \bm{\mathcal G}_{n+1}, \bm{\mathcal H}_{n+1}, \bm{\mathcal J}_{n+1}, \bm{\mathcal F}_{n+1};\bm{{\mathcal{Y}}}_{i,n}) \nonumber\\
&= \sum\limits_{v = 1}^{V} {{\left\| \bm{\mathcal{\bar S}}^{(v)} - \bm{\mathcal{\bar Q}}_{n+1}^{(v)}(\bm{\mathcal{\bar G}}_{n+1}^{(v)})^{\mathrm{T}} \right \|}_F^2} \nonumber\\
&+\lambda_1\|\bm{\mathcal {J}}_{n+1}\|^p_{\Sp} + \frac{\rho_{n}}{2}\|\bm{\mathcal {H}}_{n+1}-\bm{\mathcal {J}}_{n+1}+\frac{\bm{{\mathcal{Y}}}_{2,n}}{\rho_{n}}\|_F^2\nonumber\\
&+\lambda_2\|\bm{\mathcal {F}}_{n+1}\|^p_{\Sp} + \frac{\sigma_{n}}{2}\|\bm{\mathcal {G}}_{n+1}-\bm{\mathcal {F}}_{n+1}+\frac{\bm{{\mathcal{Y}}}_{3,n}}{\sigma_{n}}\|_F^2\nonumber\\
&+ \frac{\mu_{n}}{2}\sum\limits_{v = 1}^{V}\|\bm{\mathcal{\bar Q}}_{n+1}^{(v)}-\bm{\mathcal{\bar H}}_{n+1}^{(v)}+\frac{\bm{\mathcal{\bar Y}}_{1,n+1} ^{(v)}}{\mu_{n}}\|_F^2,
\end{align}
with each term of \eqref{eq:Ln_bdd} being nonnegative, and given the boundedness of $\bm{\mathcal L} (\bm{\mathcal Q}_{n+1}, \bm{\mathcal G}_{n+1}, \bm{\mathcal H}_{n+1}, \bm{\mathcal J}_{n+1}, \bm{\mathcal F}_{n+1}; \bm{{\mathcal{Y}}}_{i,n})$, we infer that every term in \eqref{eq:Ln_bdd} is bounded. Furthermore, the boundedness of $\|\bm{{\mathcal{J}}}_{n+1}\|^p_{\Sp}$ suggests that all singular values of $\bm{{\mathcal{J}}}_{n+1}$ are within limits, implying that $\|\bm{{\mathcal{J}}}_{n+1}\|^2_F$ (the sum of squares of singular values) is also bounded. Thus, the sequence $\{\bm{{\mathcal{J}}}_k\}$ is bounded. Similarly, the sequence $\{\bm{{\mathcal{F}}}_k\}$ is also confined within bounds.

Given that $$\bm{{\mathcal{Y}}}_{1,k+1}=\bm{{\mathcal{Y}}}_{1,k}+\mu_{k}(\bm{{\mathcal{H}}}_{k}-\bm{{\mathcal{Q}}}_{k}) \Longrightarrow \bm{{\mathcal{Q}}}_{k}=\bm{{\mathcal{H}}}_{k}+\frac{\bm{{\mathcal{Y}}}_{1,k+1}-\bm{{\mathcal{Y}}}_{1,k}}{\mu_{k}},$$ and considering the boundedness of $\bm{{\mathcal{H}}}_{k}, \bm{{\mathcal{Y}}}_{1,k}$, it becomes evident that $\bm{{\mathcal{Q}}}_{k}$ is also bounded.

Moreover, from \eqref{solveG1}, it is clear that $\|\bm{\mathcal{\bar G}}_k^{(v)}\|^2_F \leq \|(\bm{\mathcal{\bar S}}^{(v)})^\textrm{T}\|^2_F\|\bm{\mathcal{\bar Q}}_k^{(v)}\|^2_F$, thereby ensuring that $\bm{\mathcal{\bar G}}_k^{(v)}$ is bounded. Consequently, $\bm{{\mathcal{G}}}_{k}$ is also bounded.

\subsection{Proof of the 2nd part}

From the Weierstrass-Bolzano theorem, it follows that for the bounded sequence $\mathcal{P}_{k}$, there exists at least one accumulation point. We denote one such point as $\mathcal{P}^*=\{\bm{{\mathcal{H}}}^*, \bm{{\mathcal{Q}}}^*, \bm{{\mathcal{G}}}^*, \bm{{\mathcal{J}}}^*, \bm{{\mathcal{F}}}^*, \bm{{\mathcal{Y}}_1}^{*}, \bm{{\mathcal{Y}}}^*_{2}, \bm{{\mathcal{Y}}}^*_{3}\}$. Without loss of generality, we assume that the sequence $\{\mathcal{P}_{k}\}^{+\infty}_{k=1}$ converges to $P^*$.

Considering the update rule for $\bm{{\mathcal{Y}}}_1$, we have
$${\bm{\mathcal{\bar{Y}}}_{1,k+1} ^{(v)}}= {\bm{\mathcal{\bar{Y}}}_{1,k}^{(v)}} + \mu_k({\bm{\mathcal{\bar Q}}_{k}^{(v)}}-{\bm{\mathcal{\bar H}}_{k}^{(v)}})\Longrightarrow {\bm{\mathcal{\bar Q}}^{(v)*}}={\bm{\mathcal{\bar H}}^{(v)*}}.$$

Similarly, for $\bm{{\mathcal{Y}}}_2$, we obtain
$$\bm{{\mathcal{Y}}}_{2,k+1}=\bm{{\mathcal{Y}}}_{2,k}+\rho_{k}(\bm{{\mathcal{Q}}}_{k}-\bm{{\mathcal{J}}}_{k})\Longrightarrow \bm{{\mathcal{J}}}^*=\bm{{\mathcal{Q}}}^*.$$

And for $\bm{{\mathcal{Y}}}_3$, the update rule yields
$$\bm{{\mathcal{Y}}}_{3,k+1}=\bm{{\mathcal{Y}}}_{3,k}+\sigma_{k}(\bm{{\mathcal{G}}}_{k}-\bm{{\mathcal{F}}}_{k})\Longrightarrow \bm{{\mathcal{G}}}^*=\bm{{\mathcal{F}}}^*.$$

In the $\bm{\mathcal{\bar G}}^{(v)}$-subproblem \eqref{solveG1}, we deduce
$$\bm{\mathcal{\bar G}}_k^{(v)} = (\bm{\mathcal{\bar S}}^{(v)})^\textrm{T}\bm{\mathcal{\bar Q}}_k^{(v)}\Longrightarrow{\bm{\mathcal{\bar G}}}^{(v)*} = (\bm{\mathcal{\bar S}}^{(v)})^\textrm{T}\bm{\mathcal{\bar Q}}^{(v)*}.$$

In the $\bm{\mathcal {J}}$-subproblem \eqref{solveJ2}, it follows that
$$\lambda_1\nabla_{\bm{\mathcal {J}}}\|\bm{\mathcal {J}}_{k+1}\|^p_{\Sp}=\bm{{\mathcal{Y}}}_{2,k}\Longrightarrow\bm{{\mathcal{Y}}}_2^{*}=\lambda_1\nabla_{\bm{\mathcal {J}}}\|\bm{\mathcal {J}}^*\|^p_{\Sp}.$$

In the $\bm{\mathcal {F}}$-subproblem \eqref{solveF2}, we have
$$\lambda_2\nabla_{\bm{\mathcal {F}}}\|\bm{\mathcal {F}}_{k+1}\|^p_{\Sp}=\bm{{\mathcal{Y}}}_{3,k}\Longrightarrow\bm{{\mathcal{Y}}}_3^{*}=\lambda_2\nabla_{\bm{\mathcal {F}}}\|\bm{\mathcal {F}}^*\|^p_{\Sp}.$$

Therefore, we conclude that the sequences $\bm{{\mathcal{H}}}^*, \bm{{\mathcal{Q}}}^*, \bm{{\mathcal{G}}}^*, \bm{{\mathcal{J}}}^*, \bm{{\mathcal{F}}}^*, \bm{{\mathcal{Y}}_1}^{*}, \bm{{\mathcal{Y}}}^*_{2}, \bm{{\mathcal{Y}}_3}^{*}$ satisfy the KKT conditions of the Lagrange function \eqref{objective function}.


\section{Additions to the Experiments}
\subsection{Configuration and Details}
Experiments using the MSRC, HandWritten4, Mnist4, and Scene15 datasets were conducted on a laptop equipped with an Intel Core i5-8300H CPU and 16 GB RAM, utilizing Matlab R2018b. In contrast, the Reuters and NoisyMnist datasets were processed on a standard Windows 10 server, featuring dual Intel(R) Xeon(R) Gold 6230 CPUs at 2.10 GHz and 128 GB RAM, with MATLAB R2020a.

The hyper-parameters on the different datasets are as follows:
\begin{itemize}
  \item MSRC: anchor rate = 0.4, $p$ = 0.5, $\lambda_1$ = 100, $\lambda_2$ = 1.
  \item HandWritten4: anchor rate = 0.4, $p$ = 0.4, $\lambda_1$ = 5, $\lambda_2$ = 500.
  \item Mnist4: anchor rate = 0.7, $p$ = 0.5, $\lambda_1$ = 100, $\lambda_2$ = 1000.
  \item Scene15: anchor rate = 0.1, $p$ = 0.8, $\lambda_1$ = 50, $\lambda_2$ = 50.
  \item Reuters: anchor rate = 0.01, $p$ = 0.2, $\lambda_1$ = 2000, $\lambda_2$ = 2000.
  \item NoisyMnist: anchor rate = 0.004, $p$ = 0.1, $\lambda_1$ = 8000, $\lambda_2$ = 5000.
\end{itemize}

\subsection{Visualization of Experimental Results}
The visualization of the learned labels is presented in Figure \ref{labelfig}. Furthermore, we illustrate the evolution of t-SNE during its iterations using Mnist4 as an example, as depicted in Figure \ref{tsnefig}.

\begin{figure}[htbp]
	\centering
    \setlength{\abovecaptionskip}{-0.1cm}
	\includegraphics[width=0.45\linewidth]{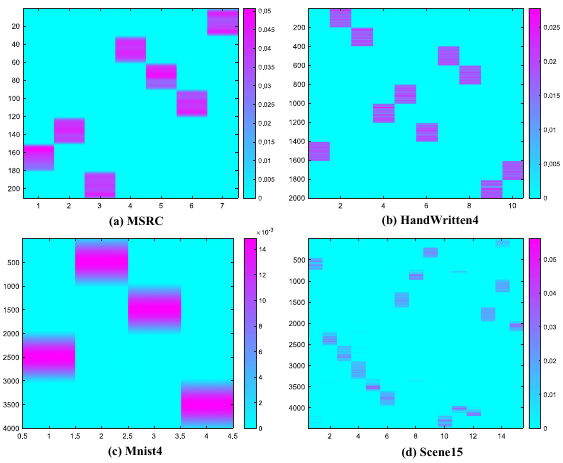}
	\caption{Clustering label visualization on MSRC, HandWritten4, Mnist4 and Scene15.}
	\label{labelfig}
    \vspace{-5mm}
\end{figure}

\begin{figure*}[htbp]
	\centering
    \setlength{\abovecaptionskip}{-0.1cm}
	\includegraphics[width=1.0\linewidth]{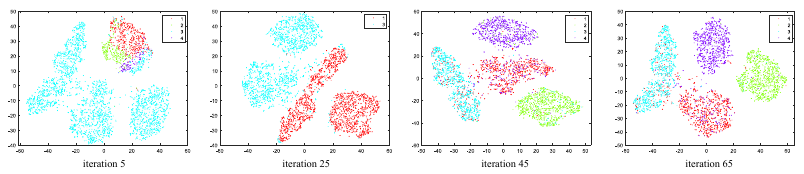}
	\caption{Iterative visualization of t-sne on Mnist4}
	\label{tsnefig}
    \vspace{-5mm}
\end{figure*}

\subsection{Experiments of Convergence}
To optimize the objective function (\ref{objective function}), we iteratively introduce three auxiliary variables: $\bm{\mathcal{Q}}$, $\bm{\mathcal{G}}$, and $\bm{\mathcal{F}}$. The convergence of the model is assessed by evaluating the differences between pairs of variables: $\bm{\mathcal{H}}$ and $\bm{\mathcal{Q}}$, $\bm{\mathcal{H}}$ and $\bm{\mathcal{J}}$, and $\bm{\mathcal{G}}$ and $\bm{\mathcal{F}}$.
Figure \ref{convergefig} depicts the variations between these variable pairs. Convergence is typically observed around the 70th iteration, where the differences stabilize, indicating a steady state. Concurrently, clustering metrics such as ACC also reach convergence, empirically demonstrating the effectiveness of our method.

\begin{figure}[htbp]
	\centering
    \setlength{\abovecaptionskip}{-0.1cm}
	\includegraphics[width=0.5\linewidth]{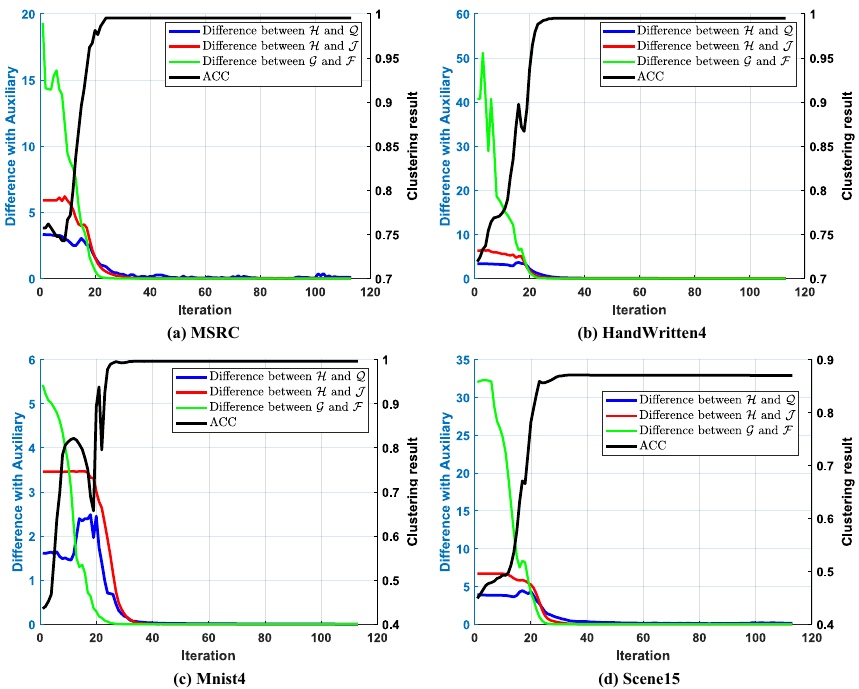}
	\caption{Convergence experiments on MSRC, HandWritten4, Mnist4 and Scene15.}
	\label{convergefig}
    \vspace{-5mm}
\end{figure}

\end{document}